\documentclass{article}

\usepackage[preprint]{neurips_2026}

\usepackage[utf8]{inputenc}
\usepackage[T1]{fontenc}
\usepackage{hyperref}
\usepackage{url}
\usepackage{booktabs}
\usepackage{amsfonts}
\usepackage{amsmath,amssymb,amsbsy,amsthm}
\usepackage{nicefrac}
\usepackage{microtype}
\usepackage{xcolor}
\usepackage{graphicx}
\usepackage{subcaption}
\usepackage{float}
\usepackage{longtable}
\usepackage{dcolumn,array}
\usepackage{cleveref}

\graphicspath{{figures/}}

\newtheorem{theorem}{Theorem}[section]

\newtheorem{remark}{Remark}[section]

\newcommand{\dm}{d_{model}}

\newcommand{\gls}[1]{\textsuperscript{\hyperref[gloss:#1]{G}}}

\newcommand{\eg}{\emph{e.g.}}
\newcommand{\ie}{\emph{i.e.}}

\newcommand{\R}{\mathbb{R}}

\newcommand{\vect}[1]{\boldsymbol{#1}}
\newcommand{\mat}[1]{\boldsymbol{#1}}

\title{Representational Capacity: Geometric Limits on Feature Representation in Transformer Language Models}

\author{%
  Alexander Guha \\
  Arizona State University \\
  \texttt{aguha6@asu.edu} \\
}

\begin{document}

\maketitle

\renewcommand{\thefootnote}{\relax}
\footnotetext{Code: \url{https://github.com/Alex-Guha/representational-capacity}}
\renewcommand{\thefootnote}{\arabic{footnote}}

\begin{abstract}
Model dimension ($d_{model}$) is a fundamental hyperparameter in transformer-based language models, yet its role in determining the geometric limits of feature representation remains under-explored.
Grounded in the Linear Representation and Superposition Hypotheses---which together propose that models encode features as near-orthogonal directions in the latent space---we develop a quantitative framework for estimating how many such directions a model's latent space can support.
We first establish the embedding matrix as a measurable proxy for the near-orthogonality constraints operating across the latent space, proposing that the boundary between meaningful token relationships and incidental similarity in the pairwise cosine similarity distribution provides a concrete estimate of the model's accepted deviation $\varepsilon$ from perfect orthogonality.
Applying this metric across dozens of open-source language models reveals two distinct classes: models with high $\varepsilon$ whose embeddings lack near-orthogonal structure, and models with low $\varepsilon$ that maintain strict near-orthogonality constraints.
We then show that the standard Johnson-Lindenstrauss lemma dramatically underestimates the packing efficiency of trained representations and derive an adjusted capacity formula in which the number of near-orthogonal directions depends on the ratio of vectors to dimensions ($k/d$) rather than the raw count alone---a single modification that reduces prediction error by two orders of magnitude with no additional free parameters.
Combining these results, we define \emph{representational capacity} as a quantitative upper bound on the number of distinguishable directions available for features and embeddings within a model's latent space.
The analysis reveals that capacity is exponentially sensitive to $\varepsilon$, and that larger models tend to favor tighter orthogonality constraints over maximizing raw capacity---a pattern compatible with several explanations (a stability--capacity trade-off, a ceiling on usable concepts, or confounds with overall model scale) that we leave open for future work.
\end{abstract}

\section{Introduction}\label{sec:introduction}

Model dimension ($d_{model}$) is one of the hyperparameters that controls a transformer-based language model's parameter count, and primarily determines the size of the embedding\gls{embeddings} and latent space\gls{latent} within the model\footnote{Definitions for terms marked with a superscript `G' can be found in the Glossary (Appendix~\ref{sec:glossary}).}.
In practice, $d_{model}$ is heuristically chosen to scale with the other hyperparameters, generally in powers of 2 to facilitate efficient GPU computation.

Naively, one might expect feature vectors of length $d_{model}$ to use each basis vector for a distinct feature.
For example, if $d_{model}=3$, the vector $[1, 0, 0]$ might represent ``Cat'', while $[0, 1, 0]$ represents ``Dog''---and $d_{model}$ would directly bound the number of features the model could work with.
In practice, neural networks have long been understood to use \emph{distributed representations}\gls{distributed}, in which each feature is represented by multiple basis dimensions being active simultaneously, and each basis dimension participates in representing many different features (\cite{bengio2014representationlearningreviewnew}).
This leads to \emph{polysemanticity}\gls{polysemanticity}, the phenomenon where individual neurons activate for multiple, seemingly unrelated inputs (\cite{olah2020zoom}).

A specific instantiation of distributed representations that has gained significant traction is the \emph{Linear Representation Hypothesis}\gls{lrh} (LRH).
Following from the introduction of the idea of linguistic regularities in the context of word embeddings (\ie~``King $-$ Man $+$ Woman $\approx$ Queen'') (\cite{kingManWomanQueen}), this idea purports that neural language models broadly tend to represent concepts and features as directions in the latent space.
Recent works have provided strong evidence for the existence of such linear feature directions in trained models.
\cite{interpretingLatentsWithSAE} employed Sparse Autoencoders (SAEs) to decompose language model activations into sparse, interpretable components, recovering feature directions including ``parts of individual names, especially last names'' and ``legal terms and court case references''.
Building on this, \cite{ScalingMonosemanticity} scaled the approach to Claude 3 Sonnet and extracted millions of monosemantic features---specific entities, code syntax, abstract concepts---including a feature direction for ``The Golden Gate Bridge'', and demonstrated that amplifying activations along such directions predictably steers model behavior.
\cite{park2024linearrepresentationhypothesisgeometry} complement these empirical findings with a theoretical analysis showing that causal interventions along specific directions can predictably manipulate model behavior, further solidifying the link between linear directions and conceptual representations.

Building on Linear Representations, the \emph{Superposition Hypothesis}\gls{superposition} accounts for polysemanticity by suggesting that neural networks leverage near-orthogonality\gls{nearortho} to represent more concepts than the number of available dimensions (\cite{toyModelsOfSuperposition}).
This idea is grounded in the Johnson-Lindenstrauss (JL) lemma\gls{jl} (\cite{johnsonLindenstraussLemma}).
As detailed in Appendix~\ref{sec:jl_lemma}, the lemma's guarantee of distance preservation implies that inner products between unit vectors are also preserved within a small accepted deviation $\varepsilon$, allowing exponentially many near-orthogonal directions to exist in high-dimensional space.
According to the Superposition Hypothesis, neural networks exploit this property by representing features as near-orthogonal directions in $\R^{\dm}$, allowing the number of representable concepts to grow exponentially relative to $\dm$.

Crucially, the SAE-based studies above not only demonstrate the existence of linear feature representations but also provide direct empirical evidence for superposition in the latent space.
\cite{ScalingMonosemanticity} extracted millions of interpretable features from Claude 3 Sonnet, a model whose latent dimension is orders of magnitude smaller than the number of recovered features.
Since these features are represented as directions in a space with far fewer dimensions than features, they must necessarily be arranged near-orthogonally---the geometric hallmark of superposition.
This establishes that near-orthogonality in the latent space is not merely a theoretical possibility but an observed property of trained models.

\paragraph{Contributions.}
This paper investigates the geometric constraints that govern how many features a transformer can represent within its $\dm$-dimensional latent space.
If models encode features as near-orthogonal directions as the Superposition Hypothesis proposes and SAE studies empirically support, then the number of such directions is bounded by geometric properties of the space: specifically, its dimension and the tolerance for deviation from perfect orthogonality.
Motivated by the relationship between tokenization and embedding structure (discussed in Section~\ref{sec:embeddings}), we analyze the similarity distribution of the embedding matrix as a method to estimate the accepted deviation $\varepsilon$ for near-orthogonality within a model's latent space.
At initialization, the embedding matrix maps orthogonal one-hot vectors into $\R^{\dm}$ via random weights, and because random vectors in high-dimensional space are near-orthogonal with high probability---the geometric property underlying the Johnson-Lindenstrauss lemma---the initial embeddings inherit this near-orthogonal structure.
Training modifies but largely preserves it: the trained distributions develop an extended right tail corresponding to lexical relationships (morphological variants of the same token) and semantic relationships (conceptually related tokens), while the bulk of unrelated token pairs remains tightly clustered near zero similarity.
The boundary between these meaningful relationships and incidental similarity---estimated as $\mu + 2\sigma$ of the distribution---provides a concrete, if heuristic, threshold for $\varepsilon$.
Applied across dozens of open-source models, this estimator reveals two distinct classes: high-$\varepsilon$ models that lack near-orthogonal embedding structure, and low-$\varepsilon$ models that maintain it.
We then show that the standard Johnson-Lindenstrauss bound dramatically underestimates the packing achieved by trained representations, and derive an empirically adjusted formula in which capacity depends on the ratio $k/d$ rather than $k$ alone---a single modification that reduces prediction error by two orders of magnitude with no additional free parameters.
Combining these, we define \emph{representational capacity}\gls{repcap} as a quantitative upper bound on the number of distinguishable directions available within the latent space, revealing that the available near-orthogonal directions constitute a shared resource among embeddings, unembeddings, and features, that capacity is exponentially sensitive to $\varepsilon$, and that larger models tend to favor tighter orthogonality over raw capacity.

\section{Embeddings}\label{sec:embeddings}

This section establishes embeddings as a measurable proxy for estimating the accepted deviation $\varepsilon$ for near-orthogonality\gls{nearortho} within a model's latent space\gls{latent}.
We characterize the post-training similarity distribution of the embedding matrix (including the lexical and semantic relationships that form its tail), propose an estimator for $\varepsilon$, and apply it across dozens of models to reveal two distinct classes.

\subsection{Tokenization and the Embedding Space}

Tokenization maps each of $V$ vocabulary tokens to a learned $\dm$-dimensional vector via an embedding matrix $\mat{E} \in \R^{V \times \dm}$\gls{embeddings}, equivalent to multiplying a one-hot vector $\vect{e}_i \in \R^V$ by $\mat{E}$ to obtain $\vect{x}_i = \mat{E}^\top \vect{e}_i$.
By construction, the input one-hot vectors are perfectly orthogonal: $\langle \vect{e}_i, \vect{e}_j \rangle = 0$ for all $i \neq j$.
Perfect orthogonality, however, requires a space with dimensionality at least equal to the number of vectors---$V$ dimensions for $V$ tokens---and since $\dm \ll V$ in practice (typical $V$ is 30,000--130,000 while $\dm$ ranges from 768--8,192), $\mat{E}$ necessarily projects these one-hot vectors into a much smaller space where strict orthogonality is no longer achievable.
At random initialization the resulting embeddings are near-orthogonal with high probability via the Johnson-Lindenstrauss lemma\gls{jl}, so the embedding matrix can be understood as a compressed representation of vocabulary space; as we will see, training largely preserves this structure.

These embeddings reside in $\R^{\dm}$, the same space occupied by all subsequent latent representations: the residual stream means embeddings, intermediate latents\gls{latents}, and features all coexist in $\R^{\dm}$ and are subject to the same geometric constraints.
We therefore hypothesize that the near-orthogonality observed in embeddings reflects properties of the broader latent space, with one qualification: as discussed in Appendix~\ref{sec:embd_unembd}, models with tied embedding and unembedding matrices exhibit different structural properties, suggesting their embeddings may not be representative.

\subsection{Post-Training Embedding Structure}

For each model, we compute the pairwise cosine similarity $\text{sim}(i,j) = \langle \vect{x}_i, \vect{x}_j \rangle / (\|\vect{x}_i\|\|\vect{x}_j\|)$ between all token embeddings.
Across many trained models, the resulting distributions are tightly centered near zero (Figure~\ref{fig:embeddings_sim_fixed}), closely resembling their random initialization; near-orthogonality is preserved through training.

\begin{figure}[htbp]
    \centering
    \includegraphics[width=0.9\textwidth]{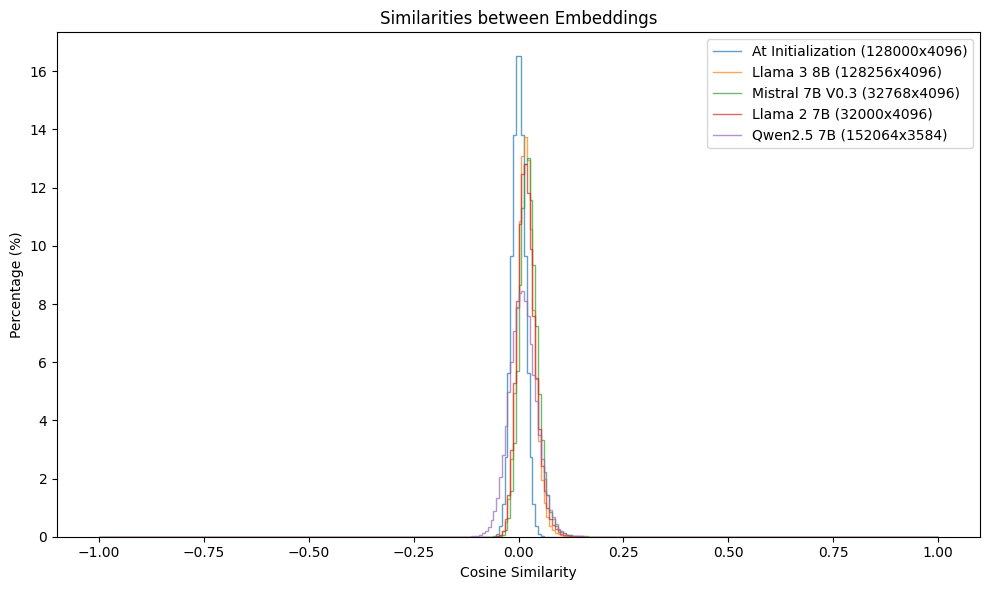}
    \caption{Distribution of pairwise cosine similarity between token embeddings across various models, demonstrating near-orthogonality.}
    \label{fig:embeddings_sim_fixed}
\end{figure}

This preservation is not coincidental: while next-token prediction does not explicitly incentivize near-orthogonality, it also does not require embeddings to collapse onto each other.
We hypothesize that the model preserves near-orthogonality as a means of maintaining a structured representational space: if embeddings \emph{lean toward}\gls{leaning} hidden feature directions, then maintaining near-orthogonality helps avoid interference between features (as suggested by the Linear Representation Hypothesis).

Two systematic deviations from random initialization are nonetheless visible on closer inspection (Figure~\ref{fig:embeddings_sim_zoomed}): the distributions are shifted slightly to the right of zero (an aversion to negative similarities, plausibly because negative pairwise similarities make the QKV projections work harder to produce useful query-key interactions), and they exhibit slightly asymmetric tails---the right tail extends further than the left, reflecting meaningful relationships between certain token pairs, which we examine next.

\begin{figure}[htbp]
    \centering
    \includegraphics[width=0.9\textwidth]{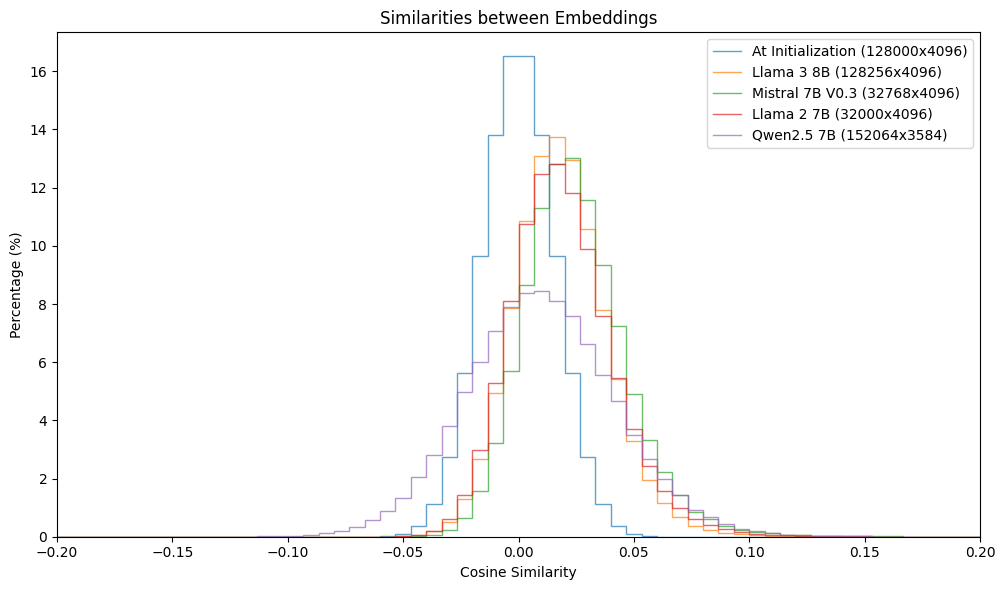}
    \caption{Zoomed view of embedding similarity distributions, revealing the right shift from zero and slightly asymmetric tails.}
    \label{fig:embeddings_sim_zoomed}
\end{figure}

\subsection{Lexical and Semantic Token Relationships}

The extended right tail in the similarity distributions corresponds to token pairs with genuine lexical or semantic relationships.
Distinguishing these from incidental similarity is essential: meaningful similarity should not count against orthogonality, while incidental similarity represents the model's tolerance for feature interference.

\begin{figure}[htbp]
    \centering
    \begin{subfigure}[b]{0.49\textwidth}
        \centering
        \includegraphics[width=\textwidth]{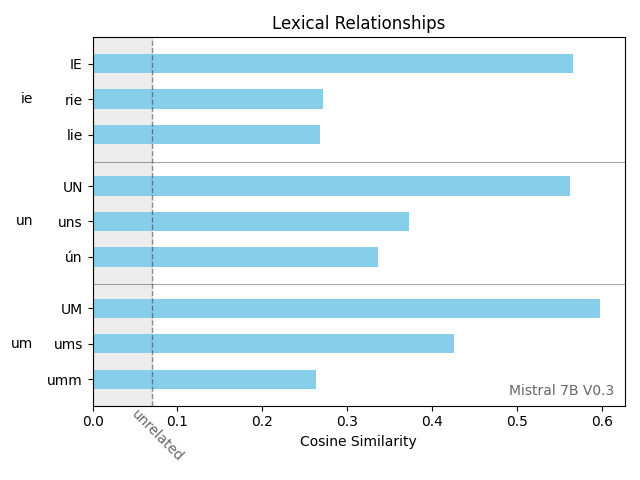}
        \caption{Lexical relationships}
    \end{subfigure}
    \begin{subfigure}[b]{0.49\textwidth}
        \centering
        \includegraphics[width=\textwidth]{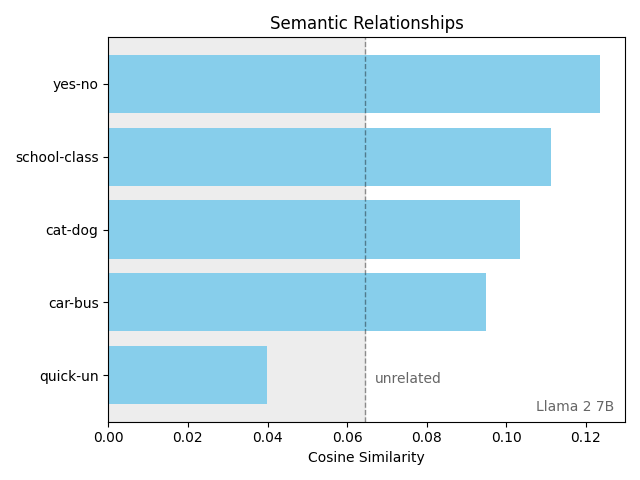}
        \caption{Semantic relationships}
    \end{subfigure}
    \caption{Examples of lexical and semantic relationships in token embeddings. Lexical relationships (a) show tokens with shared surface forms; semantic relationships (b) show conceptually related tokens without lexical overlap. The final pair (quick--un) serves as an unrelated baseline.}
    \label{fig:token_relationships}
\end{figure}

Prominent lexical relationships can be identified by examining nearest-neighbor structure.
For each token $i$, we compute its nearest neighbor $\text{nn}(i) = \operatorname*{argmax}_{j \neq i} \text{sim}(i, j)$, and call a token $k$ a \emph{primary} token if it appears frequently as the nearest neighbor of others: $|\{i : \text{nn}(i) = k\}| > \tau$ for some threshold $\tau$.
These primary tokens (Figure~\ref{fig:token_relationships}a) typically have upwards of 10 secondary tokens closer to them than to any other embedding, and they anchor clusters of capitalization or morphological variants (\eg~``cat''/``Cat''/``cats'') that account for much of the right tail.
Semantic relationships (Figure~\ref{fig:token_relationships}b) are more subtle: rather than vectors leaning toward each other directly, conceptually related tokens (\eg~``cat'' and ``dog'') appear to lean independently toward shared feature directions\gls{feature} (\eg~``animal''), yielding more limited direct similarity.
This distinction remains speculative.

\begin{figure}[htbp]
    \centering
    \includegraphics[width=0.85\textwidth]{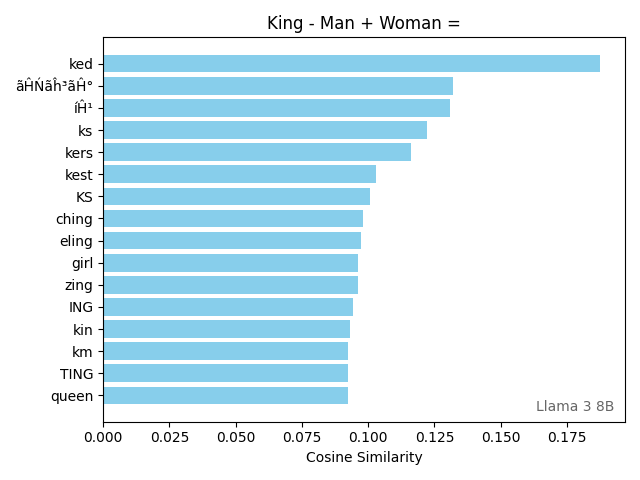}
    \caption{The top 40 most similar embeddings to the vector $\vect{x}_{\text{king}} - \vect{x}_{\text{man}} + \vect{x}_{\text{woman}}$, excluding tokens containing ``king'', ``woman'', or ``women''. The presence of ``queen'' and ``girl'' demonstrates semantic structure encoded as directions in the embedding space.}
    \label{fig:king_man_woman_relationship}
\end{figure}

The classic analogy ``King $-$ Man $+$ Woman $\approx$ Queen'' (\cite{kingManWomanQueen}) provides further evidence for this semantic structure.
Constructing a query $\vect{q} = \vect{x}_{\text{king}} - \vect{x}_{\text{man}} + \vect{x}_{\text{woman}}$ and retrieving its nearest neighbors among the embedding matrix surfaces both ``queen'' and ``girl'' within the top 40 (Figure~\ref{fig:king_man_woman_relationship}).
That ``queen'' does not rank first likely reflects the imprecision of vector arithmetic as a navigation tool through semantic feature directions; the mere presence of these semantically appropriate tokens corroborates the LRH for decoder-based models---semantic relationships are encoded as directions that can be meaningfully combined.

\subsection{Estimating and Generalizing Accepted Deviation \texorpdfstring{$\varepsilon$}{ε}}

Token pairs beyond the boundary of meaningful similarity are near-orthogonal largely as a consequence of high-dimensional geometry, and the boundary itself provides a natural estimate for the model's tolerance of incidental similarity between unrelated directions.
We propose estimating $\varepsilon$ as approximately $\mu + 2\sigma$ of the pairwise similarity distribution, where $\mu$ and $\sigma$ are its mean and standard deviation.
This threshold is a motivated heuristic rather than a formally justified bound, and is intended to mark the transition from the bulk of unrelated similarities to the relationship-driven tail.

This estimator is grounded in embedding geometry, and extending it to features warrants care: the embedding matrix is a fixed set of learned vectors, while features are directions that may be activated to varying degrees during inference.
The connection between the two is nonetheless stronger than mere shared dimensionality.
The weight matrices in attention and MLP layers---which perform the only direct interaction with latent vectors outside of normalization---function as collections of feature probes.
Each weight vector $\vect{w}_i$ extracts information from a latent $\vect{h}$ via the inner product
\begin{equation}
    \vect{h} \cdot \vect{w}_i = \|\vect{h}\| \|\vect{w}_i\| \cos\theta,
\end{equation}
which is fundamentally a scaled similarity measurement on the same $\R^{\dm}$ space, meaning weight vectors must learn directions that correspond to the features they probe.
Individual weight vectors need not align perfectly with any single feature direction---they may represent linear combinations of features, making them difficult to interpret in isolation even when the underlying feature space is well-structured.
The outputs of these transformations are then combined with the residual stream by addition, preserving the underlying geometric structure.
Nonlinear activations applied after linear transformations can be understood as response functions operating on these similarity scores rather than as modifications to the underlying geometry: directions must first be linearly distinguishable via dot product before nonlinear transformations can selectively act on them.
This framing does not preclude features that emerge from multi-layer nonlinear composition, but it does establish that each individual layer's interaction with the latent space is geometrically constrained by the same near-orthogonality measurement that governs embeddings.
Embeddings additionally offer a practical advantage as a fixed, measurable quantity, unlike intermediate representations which vary with input.
This generalization nonetheless remains an empirical hypothesis: layer normalization and the specific learned weight configurations could still impose different effective constraints on intermediate representations.

\subsection{Two Classes of Models}

To validate $\varepsilon$ as a meaningful metric, we apply it across dozens of language models and compare against $\dm$.
The analysis reveals two distinct classes (Figure~\ref{fig:model_dim_vs_ortho}; per-model values in Appendix~\ref{sec:models_analyzed}, Table~\ref{tab:model_repr_cap}), suggesting $\varepsilon$ captures a genuine structural property.

\begin{figure}[htbp]
    \centering
    \begin{subfigure}[b]{0.49\textwidth}
        \centering
        \includegraphics[width=\textwidth]{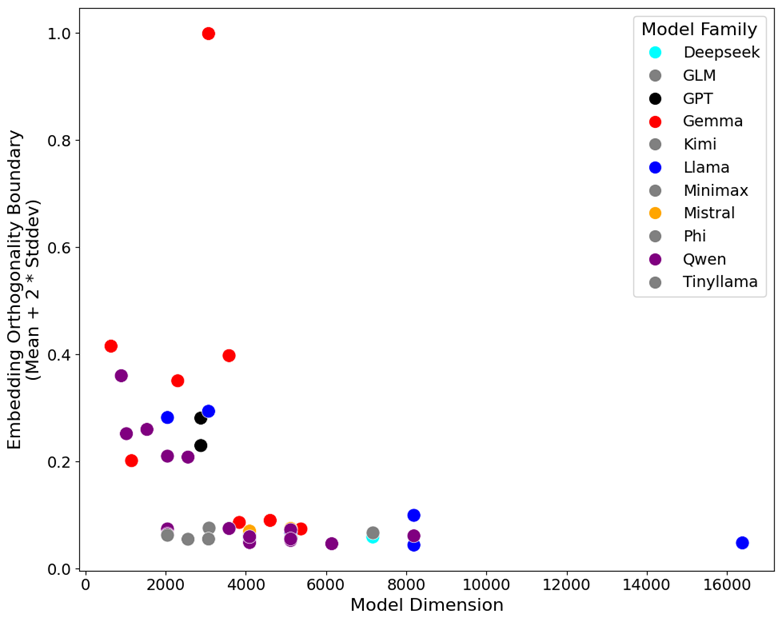}
        \caption{Two classes of models}
    \end{subfigure}
    \begin{subfigure}[b]{0.49\textwidth}
        \centering
        \includegraphics[width=\textwidth]{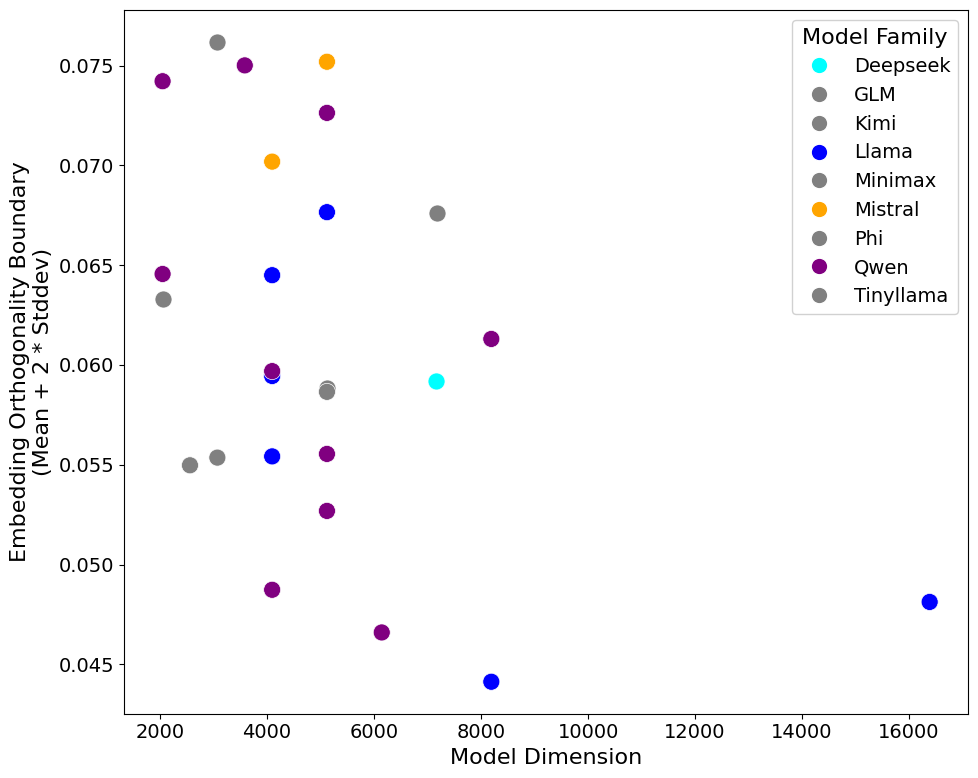}
        \caption{The second class, zoomed in}
    \end{subfigure}
    \caption{Model dimension compared to estimated $\varepsilon$ across various language models, revealing two distinct classes.}
    \label{fig:model_dim_vs_ortho}
\end{figure}

The first class consists of models with generally lower $\dm$ and wide similarity distributions, yielding $\varepsilon > 0.2$.\footnote{An extreme outlier is Gemma~7B, whose $\varepsilon$ approaches 1. We do not investigate the cause here, but it appears more consistent with idiosyncratic training or architectural factors than with a deliberate design choice.}
These models exhibit similarity-distribution means far from zero (typically 0.1--0.9), the clearest evidence that they do not leverage superposition at the embedding level: when average pairwise similarity is substantially positive, embeddings cannot serve as distinguishable, near-orthogonal directions.
Nearly all models in this class have tied embedding and unembedding matrices (see Appendix~\ref{sec:embd_unembd}).

The second class consists of models with generally higher $\dm$ and tight distributions clustered around zero ($\varepsilon < 0.1$, most around 0.09 or less), consistent with active use of superposition\gls{superposition} to pack many distinguishable directions into the latent space.
The cause of the division is not entirely clear: tying correlates with the first class but does not fully explain it, since a few tied models appear in the second class and some untied ones in the first.
The remainder of the paper focuses on the second class, where $\varepsilon$-based capacity bounds are meaningful.

\subsection{Section Summary}

This section established embeddings as a measurable proxy for the near-orthogonality constraints operating in a model's latent space.
At initialization, embedding matrices map orthogonal one-hot vectors into a near-orthogonal configuration in $\R^{\dm}$, since random vectors in high-dimensional space are approximately orthogonal with high probability.
Training modifies this initial structure only modestly: the distributions shift right (avoiding negative similarities) and develop an extended right tail (encoding lexical and semantic relationships), but the fundamental near-orthogonality persists.
The boundary between the main distribution and the relationship-driven tail---estimated as approximately $\mu + 2\sigma$---provides a concrete threshold for the accepted deviation $\varepsilon$.
Under the assumption that embeddings and features coexist in the same $\dm$-dimensional space and are subject to similar geometric constraints, this tolerance plausibly extends to the feature directions that the model learns to use, though this remains a hypothesis rather than a proven relationship.

Applying this metric across dozens of models reveals two distinct classes: high-$\varepsilon$ models that do not appear to use superposition at the embedding level, and low-$\varepsilon$ models that maintain tight near-orthogonality.
Embeddings are therefore not merely a lookup table for token representations.
They are a compressed representation of vocabulary space that \emph{leans toward}\gls{leaning} hidden feature directions, placing tokens in superposition\gls{superposition} alongside whatever features the model has learned.
Both embeddings and features draw from the same pool of near-orthogonal directions in $\R^{\dm}$, subject to the same geometric constraints quantified by $\varepsilon$.
The following section uses this $\varepsilon$ estimate to derive a quantitative bound on representational capacity.

\section{Representational Capacity}\label{sec:repr_cap}

We now convert $\varepsilon$ into a quantitative bound on \emph{representational capacity}\gls{repcap}---an upper bound on the number of near-orthogonal directions available for features, embeddings, and other learned representations within a model's latent space.
The Johnson-Lindenstrauss (JL) lemma\gls{jl} provides a natural starting point, but as we show, its random-projection assumption dramatically underestimates the packing achieved by trained models.
We derive an empirically adjusted relationship that better matches the geometry of optimized representations and use it to define and compute representational capacity.

\subsection{Near-Orthogonal Directions as a Shared Resource}\label{sec:shared_resource}

Throughout this section we use $k$ for the number of near-orthogonal directions and $d = \dm$ for the latent dimension, consistent with the JL formulation in Appendix~\ref{sec:jl_lemma}.

Near-orthogonal directions in $\R^{\dm}$ serve multiple purposes within a transformer model.
Features---the conceptual units the model has learned to represent---occupy directions in this space, as described by the Linear Representation Hypothesis\gls{lrh}.
Token embeddings and unembeddings each also require directions: the embedding matrix maps $V$ tokens into $\R^{\dm}$, and the unembedding matrix (when untied) maps latent representations back to logits over the vocabulary.
For models with separate embedding and unembedding matrices, a rough estimate of the total near-orthogonal directions required is
\begin{equation}\label{eq:shared_pool}
    k \approx 2V + k_{\text{features}},
\end{equation}
where $V$ is the vocabulary size and $k_{\text{features}}$ is the number of feature directions.
For tied models embeddings and unembeddings share directions, reducing this to approximately $k \approx V + k_{\text{features}}$.
As noted in Section~\ref{sec:embeddings}, however, most tied models fall into the high-$\varepsilon$ class and may not leverage near-orthogonality at the embedding level; the few tied models in the low-$\varepsilon$ class (\eg~Gemma) should be considered with this adjustment.

While we cannot directly measure $k_{\text{features}}$, we can work in reverse: given $\dm$ and $\varepsilon$, estimate the total $k$ available, and subtract the vocabulary contribution to bound the remaining capacity for features.

\subsection{The Johnson-Lindenstrauss Framework and Its Limits}\label{sec:jl_framework}

The JL lemma (Appendix~\ref{sec:jl_lemma}) guarantees that for $k$ unit vectors there exists a linear map to $\R^d$ preserving inner products within $\pm\varepsilon$, provided
\begin{equation}\label{eq:jl_bound}
    d \ge C \cdot \frac{\ln k}{\varepsilon^2}, \qquad \text{equivalently} \qquad k \le \exp\!\left(\frac{d \cdot \varepsilon^2}{C}\right),
\end{equation}
where $C$ depends on the construction. This implies exponential growth of $k$ with $d$ and quadratic growth with $\varepsilon$---the geometric basis for the Superposition Hypothesis\gls{superposition}.

\paragraph{The problem with random vectors.}
Applied to trained embeddings, however, this bound is wildly off.
Using the best proven $C = 8$ for Llama~2~7B ($\dm = 4096$, $\varepsilon = 0.0645$) yields
\begin{equation}
    k \le \exp\!\left(\frac{4096 \times 0.0645^2}{8}\right) = \exp(2.130) \approx 8.4
\end{equation}
near-orthogonal directions, against an actual vocabulary of 32,000.
To check whether $C = 8$ is simply too conservative, we empirically tighten it.
For each $(k, d)$ we generate $T = 1000$ independent trials of $k$ random unit vectors and record the best achievable worst-case similarity:
\begin{equation}
    \varepsilon^*_{\text{random}}(k,d) = \min_{t \in \{1,\ldots,T\}} \max_{i \neq j} \lvert\text{sim}(\vect{v}_i^{(t)}, \vect{v}_j^{(t)})\rvert.
\end{equation}
Fitting \eqref{eq:jl_bound} to this data yields $C \approx 3.029$ with excellent fit ($R^2 = 0.9985$, MAPE $1.3\%$, NRMSE $0.9\%$; Figure~\ref{fig:jl_random_fit}).
Even with this tightened constant, Llama~2~7B is bounded at only $\exp(4096 \times 0.0645^2 / 3.029) \approx 277$ directions---still nearly two orders of magnitude short of its 32,000-token vocabulary.
The issue is the assumption rather than the constant: trained embeddings are not random projections, but the result of gradient-based optimization that has discovered arrangements packing far more near-orthogonal directions than random chance allows.

\begin{figure}[htbp]
    \centering
    \includegraphics[width=0.65\textwidth]{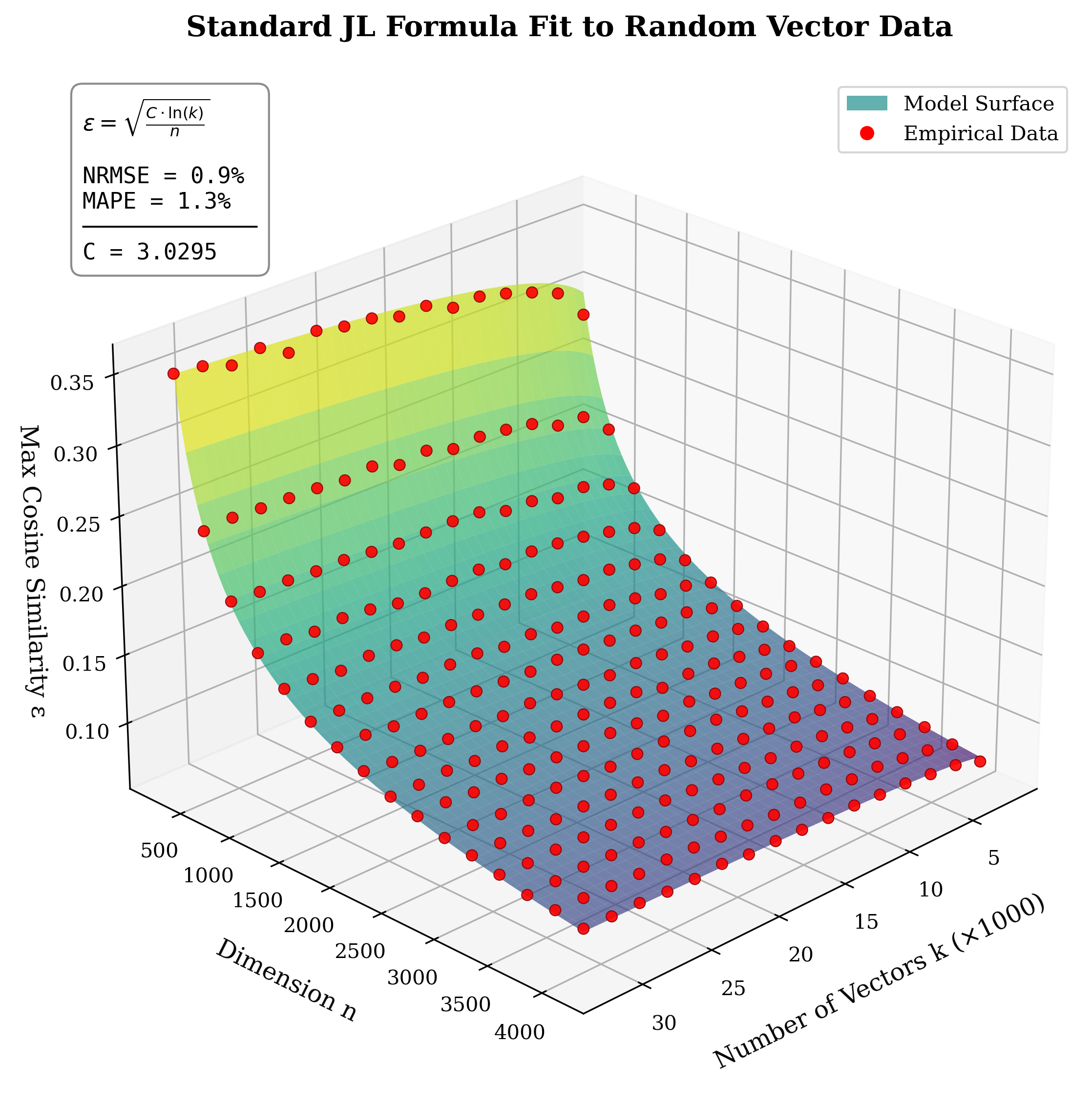}
    \caption{The standard JL relationship $\varepsilon = \sqrt{C \cdot \ln(k) / d}$ fitted to empirically generated random vector data, yielding $C \approx 3.029$ with $R^2 = 0.9985$. Even this tightened constant cannot account for the packing achieved by trained embeddings.}
    \label{fig:jl_random_fit}
\end{figure}

\subsection{An Adjusted Relationship for Optimized Vectors}\label{sec:optimizing}

To characterize what optimized vectors achieve, we initialize $k$ random unit vectors in $\R^d$ and minimize a high-exponent penalty $\mathcal{L} = \sum_{i \neq j} |G_{ij}|^p$ on the off-diagonal of the Gram matrix $\mat{G} = \mat{V}^\top\mat{V}$ ($p \in [40, 60]$, selected per scale; Adam optimizer (\cite{adam}); 5{,}000 steps).
This implicitly minimizes the maximum pairwise similarity, mimicking the tight packing that gradient descent produces during training.
We sweep $d \in \{32, 64, 128, 256, 512, 768, 1024, 1536, 2048, 2560, 3072, 3584, 4096\}$ and $k$ from 2{,}000 to 32{,}000, incrementing by 2{,}000 up to 8{,}000 for all dimensions and by 4{,}000 from 8{,}000 to 32{,}000 for $d \ge 1536$ (the coarser step reflecting the increased cost of optimizing larger configurations), and record $\varepsilon^* = \max_{i \neq j} |G_{ij}|$.

Refitting only the constant $C$ in \eqref{eq:jl_bound} to this optimized data gives a poor fit (MAPE $\approx 975\%$): the standard JL surface is too flat to track the empirical curve (Figure~\ref{fig:jl_vs_new_comparison}, left).
Among many functional-form modifications we tested, replacing $\ln(k)$ with $\ln(k/d)$ produced by far the best alignment:
\begin{equation}\label{eq:new_relationship}
    \varepsilon = \sqrt{\frac{C \cdot \ln(k/d)}{d}}, \qquad \text{equivalently} \qquad k \le d \cdot \exp\!\left(\frac{d \cdot \varepsilon^2}{C}\right).
\end{equation}
Capacity now depends on the \emph{ratio} of vectors to dimensions, not just the raw count.
With a single free parameter ($C \approx 1.293$), this fit yields $R^2 = 0.9984$, MAPE $7.9\%$---a $123\times$ MAPE reduction over standard JL on the same data, with no extra parameters.
The factor of $d$ multiplying the exponential is what allows optimized packing to dwarf what random projections can achieve.

\begin{figure}[htbp]
    \centering
    \includegraphics[width=0.95\textwidth]{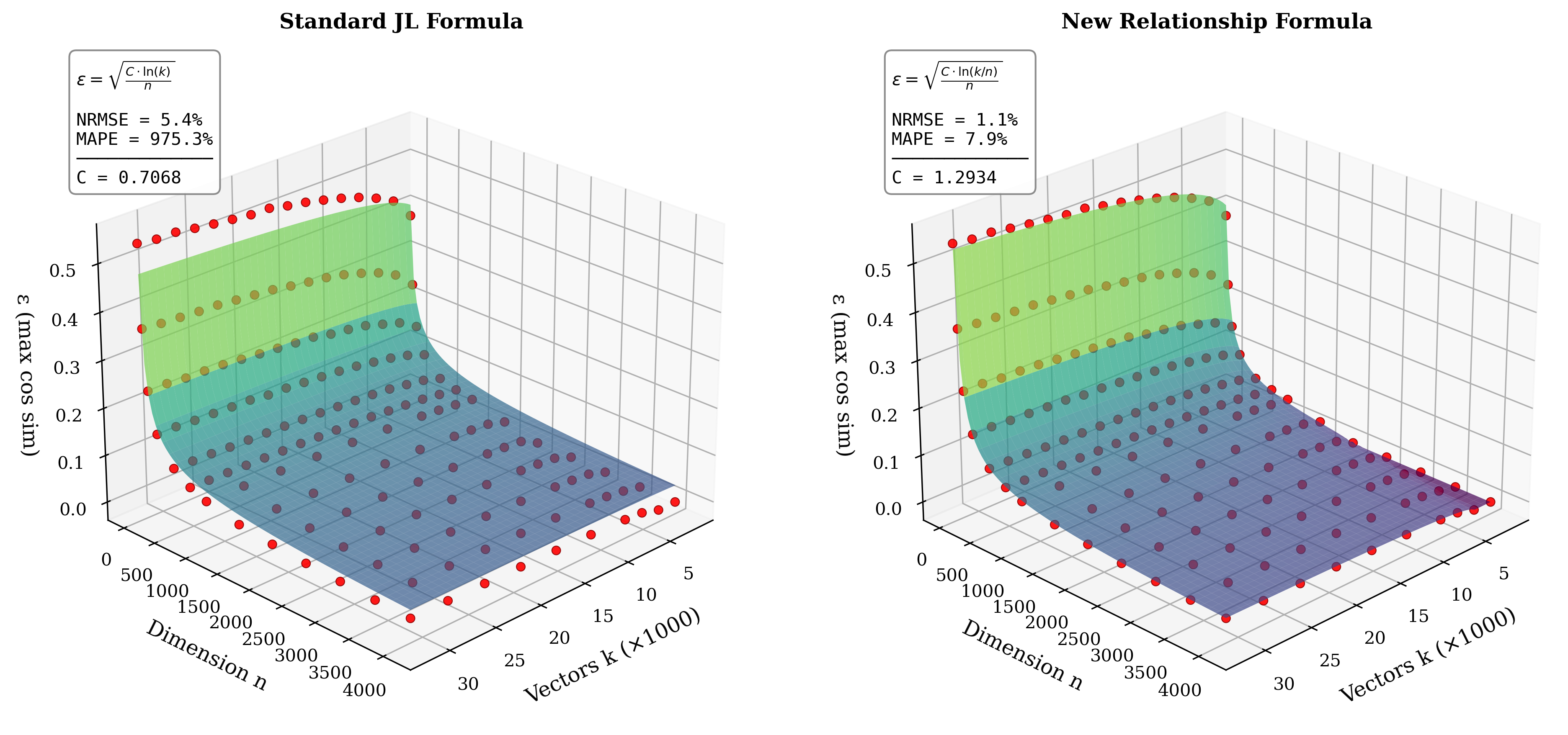}
    \caption{Standard JL formula (left) vs.\ the adjusted relationship (right) fitted to optimized vector data. Both have one free parameter; the adjusted form fits dramatically better. Red points: empirical data.}
    \label{fig:jl_vs_new_comparison}
\end{figure}

A fully parameterized variant $\varepsilon = \sqrt{C \cdot \ln(k^a/d)^b / d^c}$ confirms the structure: fitting yields $a \approx 1.07$, $c \approx 0.97$ (validating $k/d$ and $1/d$ scaling), with only modest gains in fit quality ($R^2 = 0.9998$, MAPE $5.4\%$; Figure~\ref{fig:full_parameterized_fit}).
Rearranging for $k$ gives
\begin{equation}\label{eq:k_from_full}
    k \le \left( d \cdot \exp\!\left( \left(\frac{d^c \cdot \varepsilon^2}{C}\right)^{1/b} \right) \right)^{1/a}.
\end{equation}
We use this fully parameterized form ($C = 0.458$, $a = 1.067$, $b = 1.447$, $c = 0.972$) for the per-model estimates in Appendix~\ref{sec:models_analyzed}, but the single-parameter \eqref{eq:new_relationship} captures the essential structure.

\begin{figure}[htbp]
    \centering
    \includegraphics[width=0.65\textwidth]{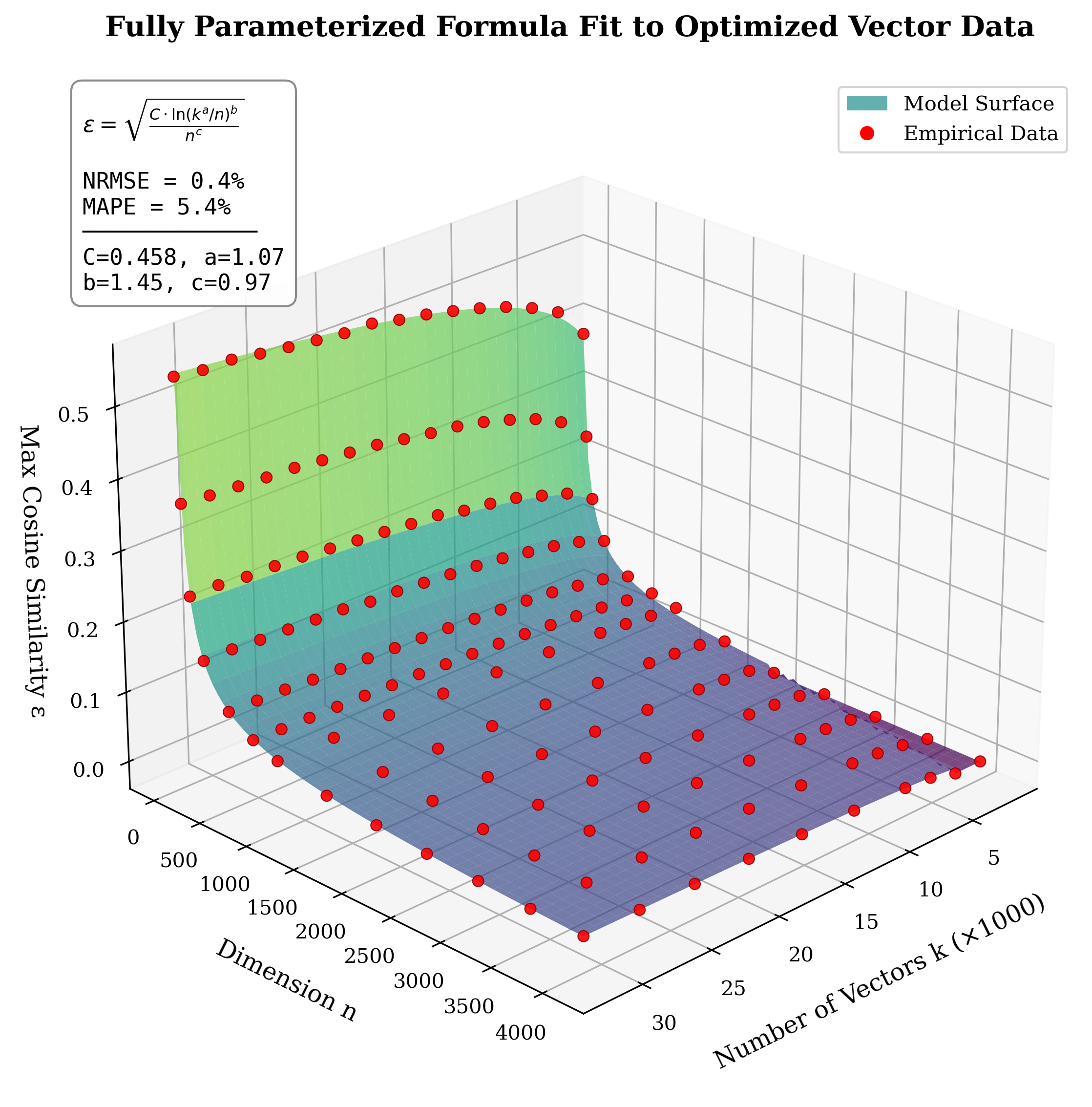}
    \caption{The fully parameterized formula $\varepsilon = \sqrt{C \cdot \ln(k^a/d)^b / d^c}$ fitted to optimized vector data, yielding $R^2 = 0.9998$. The modest improvement over the single-parameter form~\eqref{eq:new_relationship} confirms that the $k/d$ ratio is the essential structural insight rather than the extra degrees of freedom.}
    \label{fig:full_parameterized_fit}
\end{figure}

\subsection{Defining Representational Capacity}\label{sec:defining_capacity}

We define the model's \emph{representational capacity} as the upper bound on $k$ given $\dm$ and $\varepsilon$ via the adjusted relationship above.
For Llama~2~7B ($\dm = 4096$, $\varepsilon = 0.0645$) the fully parameterized form gives $k \le 3.99 \times 10^7$---roughly 40 million near-orthogonal directions, dramatically higher than the JL estimate of $\sim 277$, and comfortably accommodating both the 32,000-token vocabulary (consuming $\sim 64{,}000$ directions across embeddings and unembeddings) and the millions of features SAE studies have begun to extract.

Two caveats are essential.
First, capacity bounds geometric possibility, not realized usage: that 40 million directions \emph{can} exist does not mean the model has learned that many features, nor that its computational architecture (depth, attention heads, MLP width) could effectively utilize them.
Second, the $\varepsilon$ estimate originates from embedding geometry, and its extension to feature geometry remains a hypothesis (cf.\ Section~\ref{sec:embeddings}).
Representational capacity is therefore most useful as a \emph{relative} metric: a model with capacity $10^8$ has fundamentally more room for features than one with capacity $10^6$, regardless of how many features either actually uses.

\subsection{Capacity Across Model Classes}\label{sec:capacity_classes}

The two-class split from Section~\ref{sec:embeddings} extends to capacity.
For Class 1 models ($\varepsilon > 0.2$) the framework does not meaningfully apply: their embeddings lack near-orthogonal structure in the first place, so the formula yields vacuous bounds (\eg~$k \le 2 \times 10^{20}$ at $\dm = 2048, \varepsilon = 0.25$).
It is possible Class 1 models leverage near-orthogonality at the feature level despite their embeddings showing no such structure, but our embedding-based analysis cannot resolve this.
We focus the remaining analysis on Class 2.

For Class 2 models ($\varepsilon < 0.1$), capacities range from $\sim 10^5$ up to $\sim 10^{10}$--$10^{11}$ across the models analyzed (Appendix~\ref{sec:models_analyzed}, Table~\ref{tab:model_repr_cap}), comfortably accommodating their vocabularies.\footnote{Llama~2~70B is a notable exception, with $\varepsilon = 0.09962$ landing it at $\sim 10^{15}$; this likely reflects idiosyncratic training or architectural factors.}
The highest $\varepsilon$ we observe within Class 2 is 0.09, with no models approaching the 0.1--0.2 gap from below; this clustering suggests that models may not benefit from, or perhaps cannot sustain, deviations much larger than this threshold while maintaining the structure required for effective superposition.
Within this regime, $\varepsilon$ exerts a stronger pull on capacity than $\dm$: at fixed $\varepsilon = 0.09$, increasing $\dm$ from 2048 to 3072 gives a 30$\times$ capacity boost ($2.0\!\times\!10^7 \to 6.0\!\times\!10^8$), but at fixed $\dm = 8192$, raising $\varepsilon$ from 0.05 to 0.09 gives a six-orders-of-magnitude jump ($2.4\!\times\!10^8 \to 2.0\!\times\!10^{14}$).
The exponential dependence on $\varepsilon^2$ makes overlap tolerance, not dimension, the dominant lever.

Despite this leverage, observed models do not fully exploit it.
Figure~\ref{fig:model_dim_vs_ortho}b in Section~\ref{sec:embeddings} shows a negative correlation between $\dm$ and $\varepsilon$ within Class 2: larger models maintain tighter near-orthogonality.\footnote{We do not present a separate plot of capacity against $\dm$ or $\varepsilon$: since capacity is a deterministic function of the two, such a figure would carry no information beyond Figure~\ref{fig:model_dim_vs_ortho}b.}
This implies a trade-off between raw capacity and internal stability---larger $\varepsilon$ admits more features but more interference, while smaller $\varepsilon$ sacrifices capacity for more reliable feature retrieval---and suggests that as models scale, they favor stability over capacity maximization.
The trend rests on limited data---only a handful of analyzed models have $\dm > 6000$, making it difficult to assess whether it continues at larger scales---and the causal mechanism remains unclear.
Several explanations are plausible: there may be a practical ceiling on useful capacity beyond which additional feature directions provide diminishing returns; tighter orthogonality may be \emph{necessary} at higher $\dm$ to maintain internal stability regardless of capacity benefits; or tighter $\varepsilon$ may simply correlate with overall model size, which tends to increase alongside $\dm$ in practice.
Disentangling these factors would require controlled experiments that vary $\dm$ while holding other architectural choices constant, an expensive undertaking that remains beyond the scope of this work.

\subsection{Summary}\label{sec:repr_cap_summary}

The standard JL lemma, designed for random projections, dramatically underestimates the packing efficiency of learned representations: even with an empirically-tightened constant it bounds Llama~2~7B to fewer than 300 directions, far short of its 32,000-token vocabulary.
Replacing $\ln(k)$ with $\ln(k/d)$ in the JL relationship---a single-parameter modification motivated by direct optimization of vector arrangements---reduces prediction error by two orders of magnitude (MAPE $975\% \to 8\%$) and yields capacity bounds consistent with what is observed in trained models.
Applied across models, representational capacity is best understood as a \emph{relative} metric bounding geometric possibility rather than learned utilization: it reveals that Class 1 models fall outside the framework, that Class 2 capacities span ten orders of magnitude, and that overlap tolerance $\varepsilon$ is the dominant lever, while empirically larger models trade raw capacity for tighter near-orthogonality.
\section{Conclusion}\label{sec:conclusion}

\paragraph{Summary.}
This paper investigated the geometric constraints governing how many features a transformer-based language model can represent within its $\dm$-dimensional latent space, proceeding in four phases across two chapters.
\emph{In Section~\ref{sec:embeddings}}, we first established the embedding matrix as a measurable proxy for the near-orthogonality constraints operating across the latent space, and used the boundary between the bulk of the cosine-similarity distribution and its right tail---estimated as $\mu + 2\sigma$---as a concrete threshold for the accepted deviation $\varepsilon$.
Applying this metric across dozens of models then revealed two distinct classes with no intermediate cases: high-$\varepsilon$ models whose embeddings are not approximately orthogonal, and low-$\varepsilon$ models that maintain strict near-orthogonality.
\emph{In Section~\ref{sec:repr_cap}}, we showed that the standard Johnson--Lindenstrauss framework, designed for random projections, dramatically underestimates the packing efficiency of trained representations---predicting fewer than 300 directions for a model with a 32{,}000-token vocabulary.
By optimizing sets of vectors to mimic what gradient descent implicitly achieves, we derived an adjusted relationship in which capacity depends on the ratio $k/d$ rather than $k$ alone, reducing prediction error by two orders of magnitude with no additional free parameters; we then combined this with the per-model $\varepsilon$ estimates to define and compute \emph{representational capacity} as a quantitative upper bound on the distinguishable directions available in the latent space.
The resulting picture is one in which embeddings, unembeddings, and features draw from a shared, $\varepsilon$-bounded pool of near-orthogonal directions, and in which larger models tend to favor tighter orthogonality constraints over maximizing raw capacity---suggesting that representational stability may matter more than sheer geometric room at scale.

\paragraph{Limitations.}
Several caveats apply.
(i) The assumption that embedding geometry reflects the broader latent space is a motivated hypothesis, not a proven correspondence: embeddings are a fixed set of learned vectors, while features are dynamically activated directions, and layer normalization plus learned weights could impose different effective constraints over successive layers.
(ii) The $\mu + 2\sigma$ threshold for $\varepsilon$ is heuristic, and capacity is exponentially sensitive to $\varepsilon^2$---a shift from $\varepsilon = 0.06$ to $0.09$ moves estimated capacity by orders of magnitude---making absolute capacity values unreliable even though relative comparisons remain informative.
(iii) Architectural details we did not model (rotary positional encodings on QK subspaces, layer normalization rescaling) may impose distinct constraints on intermediate representations.
(iv) Capacity bounds geometric possibility, not realized utilization: the gap between how many directions \emph{can} exist and how many a given architecture can effectively learn and process remains unquantified.

\paragraph{Future work.}
Several directions follow naturally.
The most direct test of our central assumption would measure near-orthogonality in intermediate latent representations across inputs and layers, confirming or refuting whether embedding-derived $\varepsilon$ generalizes to the residual stream.
Correlating capacity against benchmark performance could reveal whether a ``representational scaling law'' exists, complementing existing scaling heuristics for choosing $\dm$.
A complete scaling theory will likely need to account for both representational capacity (how many features can be stored) and computational capacity (how many can be effectively processed by the network's depth, width, and attention budget).
Finally, the bifurcation into high- and low-$\varepsilon$ classes raises open questions: what architectural or training factors cause the divide, is there a critical scale at which models transition into superposition, and do Class~1 models leverage superposition internally despite their embeddings showing no evidence of it?

\paragraph{Implications for model design.}
Representational capacity currently serves as a diagnostic property of trained models, since $\varepsilon$ is not a directly controllable hyperparameter---it emerges from training dynamics and architectural choices such as embedding/unembedding tying.
The clustering of Class~2 $\varepsilon$ values below 0.09 and the observed tendency of larger-$\dm$ models to maintain tighter orthogonality together hint at a possible ceiling on useful capacity, beyond which additional geometric room provides diminishing returns.
If such a ceiling exists, $\dm$ could in principle be chosen so that the resulting capacity at $\varepsilon \lesssim 0.09$ approximately matches it, avoiding wasted dimensions; conversely, models that have not yet saturated this ceiling may benefit more from increased $\dm$ than from other forms of scaling.
Either possibility would convert the current heuristic of choosing $\dm$ in powers of two into a principled, capacity-targeted choice---though distinguishing a true ceiling from a stability requirement at scale will require controlled experiments that vary $\dm$ while holding other architectural choices fixed.

\paragraph{Closing remarks.}
We started from a simple question: given a model's latent dimension, how many features can it represent?
The answer depends not just on dimension but on how tightly the model constrains the overlap between its representations---a quantity that emerges from training rather than being set by design.
The framework here reframes $\dm$ as the determinant of a finite geometric resource that embeddings, unembeddings, and features all draw from, with the model's tolerance for overlap governing how far that resource stretches.
Whether this geometric perspective ultimately connects to model capability---whether capacity predicts what a model can learn, not just what it can store---remains the most compelling open question.

\nocite{deepseek,gemma,glm,gptoss,kimi,llama,minimax,mistral,phi,qwen,tinyllama}

\clearpage
\bibliographystyle{plainnat}
\bibliography{references}


\clearpage
\appendix

\section{Glossary}\label{sec:glossary}

The following glossary provides definitions for key terms and concepts used throughout this paper, particularly those relevant to the analysis of model dimension and representational capacity.

\bigskip

\begin{description}
    \item[Embeddings] \phantomsection\label{gloss:embeddings} The initial vector representations of tokens, produced by the embedding matrix. These are the starting points for the model's internal processing. Once they pass through the first decoder block, they are no longer considered embeddings and instead become latents.

    \item[Latents] \phantomsection\label{gloss:latents} The vector outputs of the decoder blocks, representing the evolving representations of tokens as they pass through the network. Unlike embeddings, which are static lookups from a learned matrix, latents are dynamically computed based on context.
    
    \textit{Note on terminology:} Related works often use the term ``activations'' interchangeably with what this paper refers to as latents. Here, ``activations'' specifically denotes the output vectors of individual neural layers (e.g., the query vectors resulting from $W_Q X$), whereas ``latents'' refers to the full residual stream representations after each decoder block.

    \item[Latent Space] \phantomsection\label{gloss:latent} The high-dimensional vector space where the model's internal representations (embeddings and latents) reside. The dimensionality of this space is determined by the model dimension $\dm$.

    \item[Feature/Concept] \phantomsection\label{gloss:feature} An interpretable unit of information or functionality within the model, such as a specific entity, grammatical rule, or abstract idea. Under the Linear Representation Hypothesis, features are represented as directions in the latent space.

    \item[Linear Representation Hypothesis] \phantomsection\label{gloss:lrh} A hypothesis purporting that neural language models represent concepts and features as linear directions in the latent space. Formally, for any concept $c$, there exists a direction vector $\vect{v}_c \in \R^{\dm}$ such that the activation or presence of concept $c$ in a latent vector $\vect{x} \in \R^{\dm}$ is given by the inner product $\langle \vect{x}, \vect{v}_c \rangle$.

    \item[Superposition Hypothesis] \phantomsection\label{gloss:superposition} A hypothesis proposing that neural networks leverage the properties of high-dimensional geometry---specifically near-orthogonality---to represent more concepts than the number of available dimensions. This allows the number of encoded concepts to grow exponentially with $\dm$.

    \item[Near-orthogonality] \phantomsection\label{gloss:nearortho} A geometric property describing a set of vectors that are approximately orthogonal to each other. Formally, a set of unit vectors $\mathcal{V} = \{\vect{v}_1, \dots, \vect{v}_k\} \subset \R^d$ is $\varepsilon$-nearly orthogonal if the magnitude of the inner product between any distinct pair is bounded by an accepted deviation $\varepsilon$, meaning $|\langle \vect{v}_i, \vect{v}_j \rangle| \leq \varepsilon$ for all $i \neq j$, given $0 < \varepsilon < 1$.

    \item[Johnson-Lindenstrauss (JL) Lemma] \phantomsection\label{gloss:jl} A fundamental result in high-dimensional geometry establishing that points in high-dimensional space can be embedded into lower dimensions while approximately preserving pairwise distances and inner products. Critically for the Superposition Hypothesis, the lemma's distance preservation guarantee implies that orthogonal vectors remain near-orthogonal after projection, enabling high-dimensional spaces to accommodate exponentially many near-orthogonal directions: $k \leq \exp(\frac{\varepsilon^2 d}{C})$ for some constant $C > 0$ (see full derivation in Appendix~\ref{sec:jl_lemma}).

    \item[Polysemanticity] \phantomsection\label{gloss:polysemanticity} The phenomenon observed in neural networks where individual neurons activate for multiple, distinct and often unrelated inputs or concepts (\cite{olah2020zoom}).

    \item[Representational Capacity] \phantomsection\label{gloss:repcap} A quantitative upper bound on the number of near-orthogonal directions available within a model's latent space, given its model dimension $\dm$ and the accepted deviation $\varepsilon$. Embeddings, unembeddings, and features all draw from this shared geometric resource.

    \item[Distributed Representations] \phantomsection\label{gloss:distributed} A representation paradigm where each feature is encoded as a pattern of activation across multiple basis dimensions, and each basis dimension participates in representing multiple features (\cite{bengio2014representationlearningreviewnew}).
    
    \textit{Note on terminology:} In the original distributed representations literature, the term ``feature'' referred to individual dimensions or neurons in the representation space. Under this paper's terminology---where ``feature'' is synonymous with ``concept'' (see Feature/Concept above)---distributed representations describe how each interpretable concept is spread across many basis dimensions, and conversely, each basis dimension contributes to encoding many concepts.

    \item[Leaning Toward] \phantomsection\label{gloss:leaning} A geometric relationship describing when a vector has a higher-than-expected projection onto a particular direction, where the expectation is set by the accepted deviation $\varepsilon$ for near-orthogonality. When a vector $\vect{v}$ leans toward a direction $\vect{d}$, the inner product $\langle \vect{v}, \vect{d} \rangle > \varepsilon$, indicating meaningful alignment rather than incidental similarity. This concept is central to understanding how embeddings encode information about features while remaining near-orthogonal to unrelated directions.
\end{description}

\clearpage
\section{Johnson-Lindenstrauss Lemma}\label{sec:jl_lemma}

\begin{theorem}[Johnson-Lindenstrauss Lemma, Angle Form]
Let $0 < \delta < 1$ and let $X = \{\vect{x}_1,\dots,\vect{x}_k\}$ be a set of unit vectors
in a (possibly infinite-dimensional) Hilbert space.
There exists a linear map
\[
f : \R^N \to \R^d
\quad\text{with}\quad
d = O(\delta^{-2}\log k),
\]
such that for all $i,j$,
\[
\bigl|\langle f(\vect{x}_i), f(\vect{x}_j)\rangle - \langle \vect{x}_i, \vect{x}_j\rangle\bigr|
\le 3\delta.
\]
In particular, if $\vect{x}_i \perp \vect{x}_j$, then
\[
|\langle f(\vect{x}_i), f(\vect{x}_j)\rangle| \le 2\delta,
\]
so the images are near-orthogonal.
\end{theorem}

\medskip

\begin{proof}
We begin with the standard Johnson-Lindenstrauss lemma.
There exists a linear map $f$ such that for all $\vect{u},\vect{v} \in X \cup \{\vect{0}\}$,
\begin{equation}\label{eq:jl}
(1-\delta)\|\vect{u}-\vect{v}\|^2 \le \|f(\vect{u})-f(\vect{v})\|^2 \le (1+\delta)\|\vect{u}-\vect{v}\|^2.
\end{equation}

\medskip

\noindent
\textbf{Step 1: Norm preservation.}
Setting $\vect{v} = \vect{0}$ in \eqref{eq:jl} yields
\[
(1-\delta)\|\vect{u}\|^2 \le \|f(\vect{u})\|^2 \le (1+\delta)\|\vect{u}\|^2.
\]
Since each $\vect{x}_i$ is a unit vector,
\begin{equation}\label{eq:norm}
1-\delta \le \|f(\vect{x}_i)\|^2 \le 1+\delta.
\end{equation}

\medskip

\noindent
\textbf{Step 2: Distance between two unit vectors.}
For any unit vectors $\vect{u},\vect{v}$,
\[
\|\vect{u}-\vect{v}\|^2 = \|\vect{u}\|^2 + \|\vect{v}\|^2 - 2\langle \vect{u},\vect{v}\rangle
= 2 - 2\langle \vect{u},\vect{v}\rangle.
\]
Applying \eqref{eq:jl},
\begin{equation}\label{eq:dist}
2(1-\delta)(1-\langle \vect{u},\vect{v}\rangle)
\le \|f(\vect{u})-f(\vect{v})\|^2
\le 2(1+\delta)(1-\langle \vect{u},\vect{v}\rangle).
\end{equation}

\medskip

\noindent
\textbf{Step 3: Recover inner products using polarization.}
For any vectors $\vect{a},\vect{b}$,
\[
\langle \vect{a},\vect{b}\rangle
= \frac{\|\vect{a}\|^2 + \|\vect{b}\|^2 - \|\vect{a}-\vect{b}\|^2}{2}.
\]
Applying this to $f(\vect{u}), f(\vect{v})$,
\begin{equation}\label{eq:polar}
\langle f(\vect{u}),f(\vect{v})\rangle
= \frac{\|f(\vect{u})\|^2 + \|f(\vect{v})\|^2 - \|f(\vect{u})-f(\vect{v})\|^2}{2}.
\end{equation}

\medskip

\noindent
\textbf{Step 4: Bounding the inner product distortion.}
Using \eqref{eq:norm} and the upper bound from \eqref{eq:dist},
\begin{align*}
\langle f(\vect{u}),f(\vect{v})\rangle
&\ge \frac{2(1-\delta) - 2(1+\delta)(1-\langle \vect{u},\vect{v}\rangle)}{2}\\
&= (1-\delta) - (1+\delta) + (1+\delta)\langle \vect{u},\vect{v}\rangle\\
&= (1+\delta)\langle \vect{u},\vect{v}\rangle - 2\delta.
\end{align*}
Similarly, using the lower bound from \eqref{eq:dist},
\begin{align*}
\langle f(\vect{u}),f(\vect{v})\rangle
&\le \frac{2(1+\delta) - 2(1-\delta)(1-\langle \vect{u},\vect{v}\rangle)}{2}\\
&= (1+\delta) - (1-\delta) + (1-\delta)\langle \vect{u},\vect{v}\rangle\\
&= (1-\delta)\langle \vect{u},\vect{v}\rangle + 2\delta.
\end{align*}
Thus,
\[
|\langle f(\vect{u}),f(\vect{v})\rangle - \langle \vect{u},\vect{v}\rangle|
\le 2\delta + \delta|\langle \vect{u},\vect{v}\rangle|
\le 3\delta,
\]
where the last inequality uses $|\langle \vect{u},\vect{v}\rangle| \le 1$ for unit vectors.

\medskip

\noindent
\textbf{Step 5: Near-orthogonality.}
If $\vect{u} \perp \vect{v}$, then $\langle \vect{u},\vect{v}\rangle = 0$, and the bounds in Step~4 give
\[
|\langle f(\vect{u}),f(\vect{v})\rangle| \le 2\delta.
\]
Hence the images are near-orthogonal.

\medskip

\noindent
\textbf{Step 6: Angles.}
Since $\|f(\vect{u})\|,\|f(\vect{v})\| \in [\sqrt{1-\delta},\sqrt{1+\delta}]$,
the cosine of the angle $\theta'$ between $f(\vect{u})$ and $f(\vect{v})$ satisfies
\[
|\cos\theta'|
= \frac{|\langle f(\vect{u}),f(\vect{v})\rangle|}{\|f(\vect{u})\|\|f(\vect{v})\|}
\le \frac{3\delta}{1-\delta}
= O(\delta).
\]

This completes the proof.
\end{proof}

\bigskip

\begin{remark}
The dimension requirement $d = O(\delta^{-2}\log k)$ can be written as $d \ge C' \cdot \frac{\ln k}{\delta^2}$ for some constant $C' > 0$.
For orthogonal vectors, Step~5 shows the inner product deviation is bounded by $2\delta$.
Setting $\varepsilon = 2\delta$---where $\varepsilon$ is the accepted deviation on near-orthogonality used throughout this paper---gives
\[
d \ge C \cdot \frac{\ln k}{\varepsilon^2}
\]
for an adjusted constant $C$ that absorbs the factor of 4.
Rearranging yields
\[
k \le \exp\!\left(\frac{\varepsilon^2 d}{C}\right),
\]
which bounds the number of $\varepsilon$-near-orthogonal unit vectors that a random projection can produce in $\R^d$.
This exponential growth in $k$ with respect to $d$ is central to the Superposition Hypothesis.

The value of $C$ depends on the specific probabilistic construction used; Section~\ref{sec:repr_cap} discusses both the theoretical value $C = 8$ from the literature and an empirically fitted value for optimized vectors.
\end{remark}

\clearpage
\section{Embedding and Unembedding Relationship}\label{sec:embd_unembd}

This appendix examines the relationship between embedding and unembedding matrices, which informs the interpretation of embedding-based $\varepsilon$ estimates in Section~\ref{sec:embeddings} and motivates the qualification noted there for tied models.

At the output layer, the final latent representations are converted to logits via the unembedding matrix $\mat{U} \in \R^{V \times \dm}$, with the logit for token $i$ given by the inner product $\langle \vect{h}, \vect{u}_i \rangle$ between the final latent $\vect{h}$ and the corresponding row $\vect{u}_i$.
In \emph{untied} models, $\mat{E}$ and $\mat{U}$ are learned independently and may differ significantly; in \emph{tied} models the same matrix serves both roles.

\begin{figure}[htbp]
    \centering
    \includegraphics[width=0.9\textwidth]{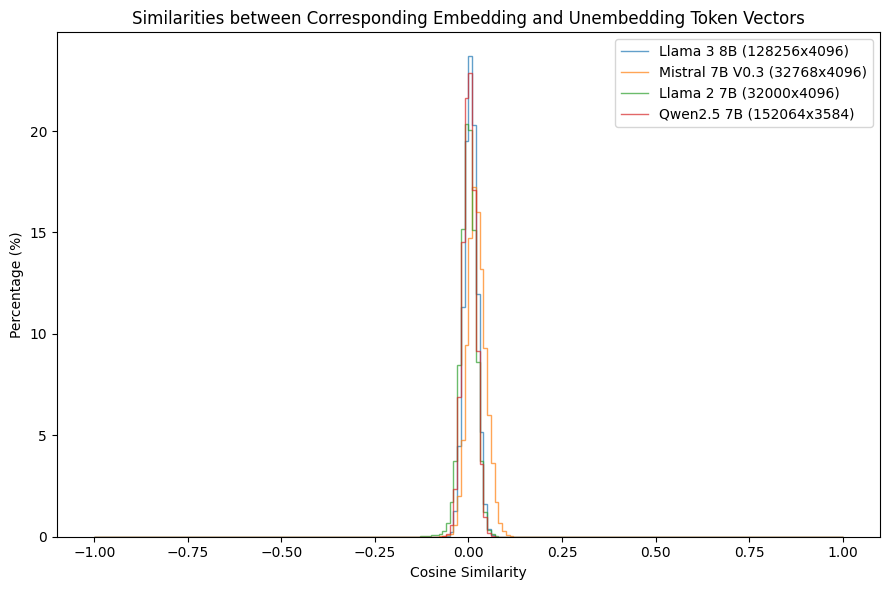}
    \caption{Cosine similarity between corresponding token embeddings and unembeddings for several models with untied matrices.}
    \label{fig:embd_vs_unembd}
\end{figure}

Comparing the embedding and unembedding vectors for each token in untied models (Figure~\ref{fig:embd_vs_unembd}) reveals that they are largely orthogonal: the model transforms representations from an input embedding space into a distinct output unembedding space over the course of the residual stream.
The similarity distributions of embeddings (Section~\ref{sec:embeddings}) and unembeddings (Figure~\ref{fig:unembedding_similarities}) also differ noticeably, likely reflecting the different roles these matrices play---embeddings provide a useful starting representation for contextualization, while unembeddings must support accurate next-token prediction.

\begin{figure}[htbp]
    \centering
    \includegraphics[width=0.9\textwidth]{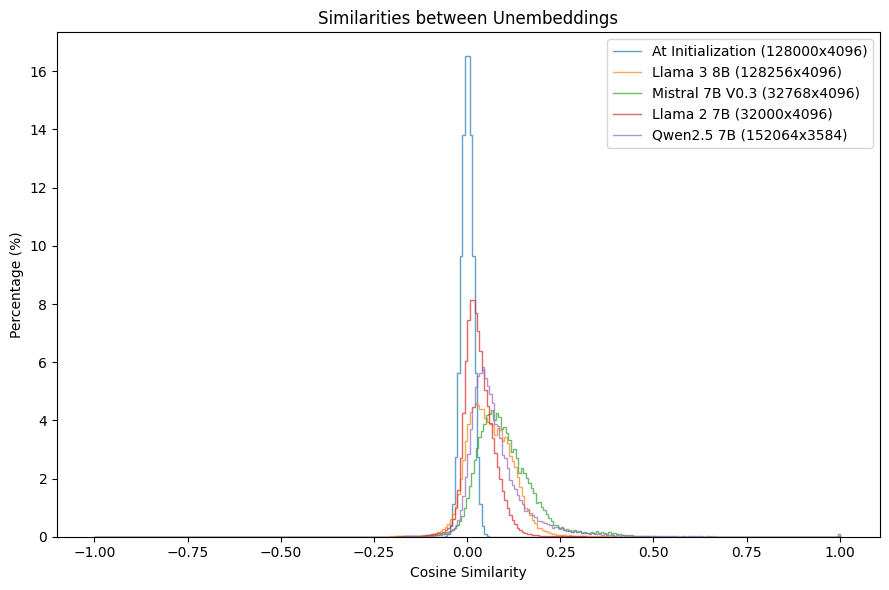}
    \caption{Distribution of cosine similarity between each token unembedding and all others.}
    \label{fig:unembedding_similarities}
\end{figure}

In tied models the embedding similarity distributions follow the unembedding pattern, since the two are the same matrix.
This is relevant to the two-class split in Section~\ref{sec:embeddings}: nearly all Class~1 (high-$\varepsilon$) models are tied, suggesting that forcing one matrix to satisfy both objectives discourages the near-orthogonal structure that untied embeddings exhibit.
Tying does not fully account for the divide---a few tied models maintain low $\varepsilon$ and some untied ones do not---so the structural cause remains open.

\clearpage
\section{Models Analyzed}\label{sec:models_analyzed}

Per-model values of $\dm$, the estimated accepted deviation $\varepsilon$ ($\mu + 2\sigma$ of the pairwise embedding cosine similarity distribution), and the resulting representational capacity computed via the fully parameterized form of \eqref{eq:new_relationship} (Section~\ref{sec:optimizing}).
Class~1 models ($\varepsilon > 0.2$) are marked indeterminate: their embeddings lack near-orthogonal structure, so the bound is vacuous (cf.\ Section~\ref{sec:capacity_classes}).
Across the Class~2 models, the bulk of capacities top out near $\sim 10^{10}$--$10^{11}$; the lone exception is Llama~2~70B, whose $\varepsilon = 0.09962$ sits at the very edge of the Class~1 boundary and lands its capacity at $\sim 10^{15}$.

\begin{longtable}{l r r r}
    \caption{Per-model $\dm$, accepted deviation $\varepsilon$, and representational capacity.}
    \label{tab:model_repr_cap} \\
    \hline
    \textbf{Model Name} & \textbf{$\dm$} & \textbf{$\varepsilon$} & \textbf{Rep. Capacity} \\
    \hline
    \endfirsthead

    \hline
    \textbf{Model Name} & \textbf{$\dm$} & \textbf{$\varepsilon$} & \textbf{Rep. Capacity} \\
    \hline
    \endhead

    \hline
    \endfoot

    \endlastfoot

    DeepSeek V3.2 Exp & 7168 & 0.05917 & $1.16\mathrm{e}{9}$ \\
    DeepSeek R1 & 7168 & 0.05917 & $1.16\mathrm{e}{9}$ \\
    \hline
    Gemma 7B & 3072 & 0.99918 & indeterminate \\
    Gemma 2 2B & 2304 & 0.35084 & indeterminate \\
    Gemma 2 9B & 3584 & 0.39785 & indeterminate \\
    Gemma 2 27B & 4608 & 0.09013 & $4.77\mathrm{e}{10}$ \\
    Gemma 3 270M & 640 & 0.41554 & indeterminate \\
    Gemma 3 1B & 1152 & 0.20164 & indeterminate \\
    Gemma 3 12B & 3840 & 0.08635 & $2.51\mathrm{e}{9}$ \\
    Gemma 3 27B & 5376 & 0.07421 & $4.36\mathrm{e}{9}$ \\
    \hline
    GLM 4.6 & 5120 & 0.05882 & $6.14\mathrm{e}{7}$ \\
    \hline
    GPT OSS 120B & 2880 & 0.22984 & indeterminate \\
    GPT OSS 20B & 2880 & 0.28120 & indeterminate \\
    \hline
    Kimi K2 & 7168 & 0.06763 & $1.47\mathrm{e}{10}$ \\
    \hline
    Llama 2 7B & 4096 & 0.06450 & $3.99\mathrm{e}{7}$ \\
    Llama 2 13B & 5120 & 0.06766 & $5.11\mathrm{e}{8}$ \\
    Llama 2 70B & 8192 & 0.09962 & $8.15\mathrm{e}{15}$ \\
    Llama 3 8B & 4096 & 0.05945 & $1.42\mathrm{e}{7}$ \\
    Llama 3.1 8B & 4096 & 0.05542 & $6.37\mathrm{e}{6}$ \\
    Llama 3.1 70B & 8192 & 0.04413 & $4.39\mathrm{e}{7}$ \\
    Llama 3.1 405B & 16384 & 0.04813 & $1.22\mathrm{e}{11}$ \\
    Llama 3.2 1B & 2048 & 0.28223 & indeterminate \\
    Llama 3.2 3B & 3072 & 0.29395 & indeterminate \\
    \hline
    MiniMax M2 & 3072 & 0.07618 & $4.39\mathrm{e}{7}$ \\
    \hline
    Mistral Small 3.2 24B & 5120 & 0.07520 & $3.39\mathrm{e}{9}$ \\
    Mistral 7B Instruct v0.3 & 4096 & 0.07019 & $1.33\mathrm{e}{8}$ \\
    \hline
    Phi 2 & 2560 & 0.05497 & $4.57\mathrm{e}{5}$ \\
    Phi 3 mini 128k & 3072 & 0.05536 & $1.21\mathrm{e}{6}$ \\
    Phi 4 & 5120 & 0.05865 & $5.91\mathrm{e}{7}$ \\
    \hline
    Qwen 2.5 0.5B & 896 & 0.36035 & indeterminate \\
    Qwen 2.5 1.5B & 1536 & 0.25998 & indeterminate \\
    Qwen 2.5 3B & 2048 & 0.21014 & indeterminate \\
    Qwen 2.5 7B & 3584 & 0.07501 & $1.20\mathrm{e}{8}$ \\
    Qwen 2.5 14B & 5120 & 0.05269 & $1.51\mathrm{e}{7}$ \\
    Qwen 2.5 32B & 5120 & 0.05554 & $2.88\mathrm{e}{7}$ \\
    Qwen 2.5 72B & 8192 & 0.06131 & $8.46\mathrm{e}{9}$ \\
    Qwen 3 0.6B & 1024 & 0.25204 & indeterminate \\
    Qwen 3 4B & 2560 & 0.20834 & indeterminate \\
    Qwen 3 8B & 4096 & 0.05969 & $1.49\mathrm{e}{7}$ \\
    Qwen 3 32B & 5120 & 0.07263 & $1.77\mathrm{e}{9}$ \\
    Qwen 3 30B A3B & 2048 & 0.06456 & $5.67\mathrm{e}{5}$ \\
    Qwen 3 Next 80B A3B & 2048 & 0.07422 & $2.07\mathrm{e}{6}$ \\
    Qwen 3 235B A22B & 4096 & 0.04874 & $1.77\mathrm{e}{6}$ \\
    Qwen 3 Coder 480B & 6144 & 0.04660 & $1.21\mathrm{e}{7}$ \\
    \hline
    TinyLlama 1.1B & 2048 & 0.06329 & $4.81\mathrm{e}{5}$ \\
    \hline
\end{longtable}

\end{document}